\documentclass[a4paper,twoside]{article}

\usepackage{graphicx}
\usepackage{subcaption}
\usepackage{calc}
\usepackage{amssymb}
\usepackage{amstext}
\usepackage{amsmath}
\usepackage{amsthm}
\usepackage{multicol}
\usepackage{pslatex}
\usepackage{algorithm2e}
\usepackage[bottom]{footmisc}
\usepackage{array,booktabs}
\usepackage{multirow}
\newtheorem{theorem}{Theorem}
\newtheorem{lemma}{Lemma}
\newtheorem{corollary}{Corollary}[lemma]
\usepackage{pifont}

\usepackage{natbib}
\usepackage{verbatim}

\usepackage[hidelinks]{hyperref}

\usepackage{SCITEPRESS}     

\begin{document}

\title{Comparing Euclidean and Hyperbolic K-Means for Generalized Category Discovery}

\author{\authorname{Mohamad Dalal\sup{1}\orcidAuthor{0009-0006-3857-4220}, Thomas B. Moeslund\sup{1,2}\orcidAuthor{0000-0001-7584-5209}, and Joakim Bruslund Haurum\sup{2,3}\orcidAuthor{0000-0002-0544-0422}}
\affiliation{\sup{1}Visual Analysis and Perception Lab, Aalborg University, Aalborg, Denmark}
\affiliation{\sup{2}Pioneer Centre for AI, Denmark, \sup{3}Center for Software Technology, University of Southern Denmark, Vejle, Denmark}
\email{\{moda, tbm\}@create.aau.dk, jhau@mmmi.sdu.dk}}



\abstract{Hyperbolic representation learning has been widely used to extract implicit hierarchies within data, and recently it has found its way to the open-world classification task of Generalized Category Discovery (GCD).
However, prior hyperbolic GCD methods only use hyperbolic geometry for representation learning and transform back to Euclidean geometry when clustering. We hypothesize this is suboptimal.
Therefore, we present Hyperbolic Clustered GCD (HC-GCD), which learns embeddings in the Lorentz Hyperboloid model of hyperbolic geometry, and clusters these embeddings directly in hyperbolic space using a hyperbolic K-Means algorithm. 
We test our model on the Semantic Shift Benchmark datasets, and demonstrate that HC-GCD is on par with the previous state-of-the-art hyperbolic GCD method. Furthermore, we show that using hyperbolic K-Means leads to better accuracy than Euclidean K-Means. 
We carry out ablation studies showing that clipping the norm of the Euclidean embeddings leads to decreased accuracy in clustering unseen classes, and increased accuracy for seen classes, while the overall accuracy is dataset dependent. We also show that using hyperbolic K-Means leads to more consistent clusters when varying the label granularity. 
}

\onecolumn \maketitle \normalsize \setcounter{footnote}{0} \vfill

\section{\uppercase{Introduction}}
\label{sec:introduction}

Image classification has seen a renascence in recent years, with a shift in focus onto open-world classification. One such task is Generalized Category Discovery (GCD) \citep{Vaze_2022_CVPR}, which poses the challenge of being able to cluster seen and unseen classes simultaneously.

One approach to better classify unseen classes is to utilize the latent hierarchies present within the data, as hierarchies are implicitly present in most data \citep{Ge_2023_CVPR,NEURIPS2021_291d43c6}. Hyperbolic representation learning is often used to learn these hierarchies, and recently some works have explored the use of hyperbolic representation learning to solve the GCD task \citep{Liu2025HypCD, HIDISC}. 

However, all hyperbolic learned non-parametric methods resort to performing clustering on the Euclidean embeddings instead of hyperbolic \citep{Liu2025HypCD, HIDISC}. This can lead to distortion of the learned hierarchies, as Euclidean geometry is not optimal for representing these hierarchies \citep{Gromov1987}. This work aims to explore clustering directly in hyperbolic geometry, by introducing a hyperbolic K-Means algorithm and testing its performance on the Semantic Shift Benchmark \citep{vaze2022the}, the most common GCD benchmark.

To achieve this we introduce Hyperbolic Clustered GCD (HC-GCD), which is a GCD model trained using contrastive learning in the Lorentz Hyperboloid model of hyperbolic geometry. This was chosen instead of the Poincaré model due to the existence of a closed form centroid, which is later used to create clusters using a hyperbolic adaptation of the semi-supervised K-Means algorithm, seen in Figure \ref{fig:Lorentz-KMeans}.\footnote{The code for this paper can be found at \href{https://github.com/MohamadDalal/HC-GCD}{https://github.com/MohamadDalal/HC-GCD}}

\begin{figure}[t]
    \centering
    \includegraphics[width=\linewidth]{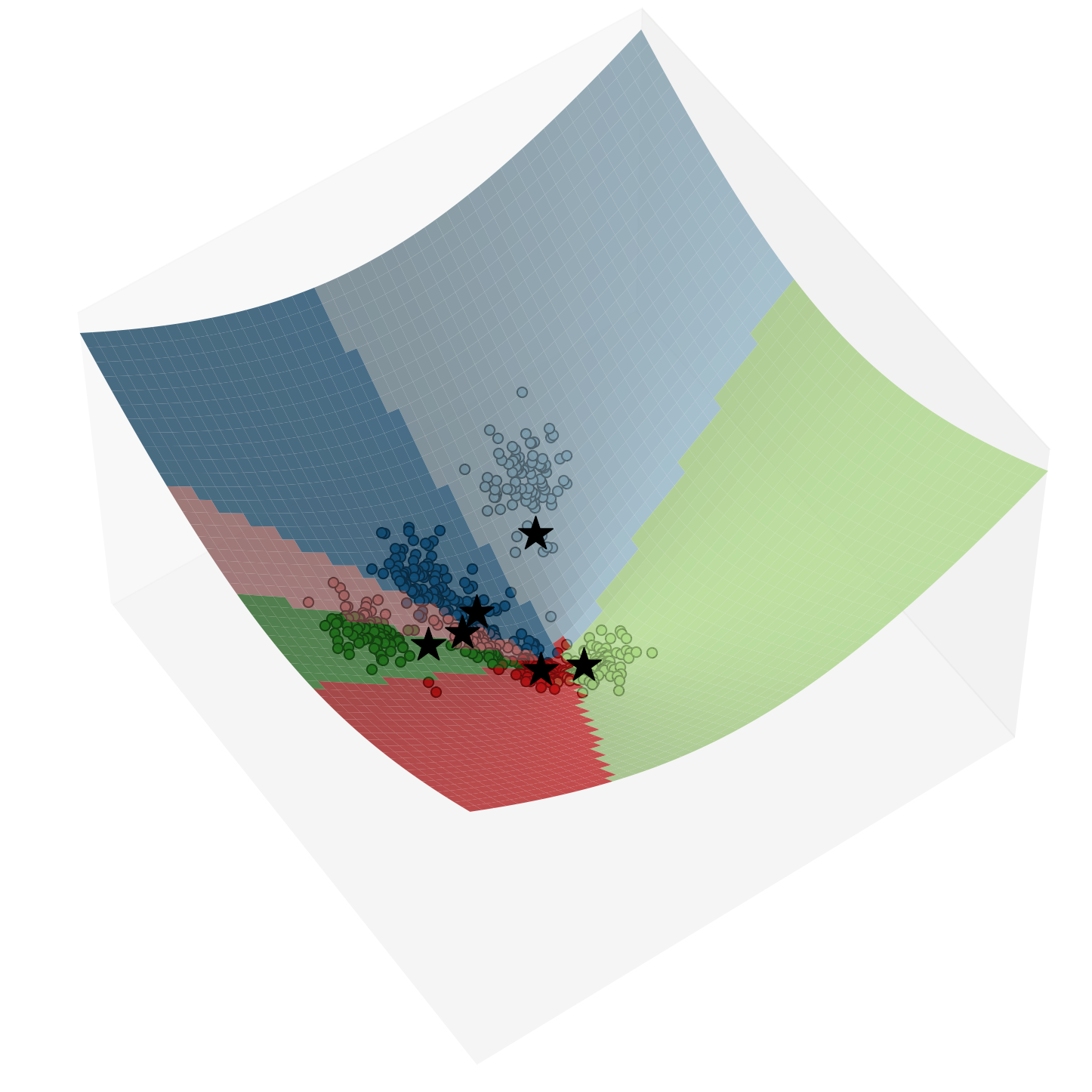}
    \caption{Hyperbolic semi-supervised K-Means in the Lorentz Hyperboloid model of hyperbolic geometry. The centroids are represented as black stars.}
    \label{fig:Lorentz-KMeans}
\end{figure}

Our contributions are as follows:
\begin{itemize}
    \item Creating a hyperbolic K-Means algorithm, and showing its benefits on the GCD task.
    \item Mathematically proving that the Einstein Midpoint is equivalent to the Lorentz Centroid.
    \item Demonstrating that hyperbolic clustering leads to better implicit feature hierarchies, through a label granularity analysis.
\end{itemize}

\section{\uppercase{Related Works}}

\paragraph{Hyperbolic Representation Learning}
Hyperbolic geometry is the geometry of surfaces with constant negative Gaussian curvature. This is in contrast to Euclidean geometry for surfaces with zero Gaussian curvature and hyperspherical geometry for surfaces with constant positive Gaussian curvature \citep{ratcliffe1994foundations}. This leads to differing properties, such as the existence of at least two lines parallel to a stationary line in contrast to a maximum of one in Euclidean geometry, and the exponential growth of areas of spheres in contrast to polynomial growth in Euclidean geometry \citep{PhysRevE.82.036106}. This exponential growth makes hyperbolic geometry particularly suitable for representing hierarchies, which also grow exponentially with increased depth \citep{Gromov1987}. Hyperbolic geometry can be represented by many different models. However, the three most common are the Poincaré Hypersphere, the Klein Hypersphere and the Lorentz Hyperboloid.

Hyperbolic geometry's affinity to representing hierarchies led to it being explored for learning representations of many types of hierarchical data. An example from the text domain is skip-gram by \citet{LeimeisterSkipGram}, where they generate hyperbolic text embeddings by adding a hyperbolic distance objective to Word2Vec \citep{mikolov2013efficientestimationwordrepresentations}. While from the vision domain \citet{Ge_2023_CVPR} learn the hierarchies between scenes and the objects within them, while \citet{Atigh_2022_CVPR} and \citet{Suris_2021_CVPR} use the norm of hyperbolic embeddings to calculate the uncertainty of predictions.
These methods are built upon a Euclidean backbone where the final features are lifted into hyperbolic geometry using the exponential map. There has, however, also been recent work on making fully hyperbolic neural networks \citep{PoincareResNet}. In the intersection of the two modalities is MERU \citep{pmlr-v202-desai23a}, which learns vision-text embeddings in the Lorentz Hyperboloid model of hyperbolic geometry. This is done using hyperbolic distance contrastive loss and hyperbolic entailment loss to learn the hierarchy present between an image and its textual description.

More details can be found in the surveys by \citet{Hyp-DL-Survey}, and \citet{mettes2024hyperbolic}.

\paragraph{Generalized Category Discovery} In 2022 \citet{Vaze_2022_CVPR} proposed the \textit{Generalized Category Discovery} task, where a mix of labeled and unlabeled data has to be clustered according to both seen and unseen classes. \citet{Vaze_2022_CVPR} originally proposed a two stage approach consisting of a representation learning step (self-supervised contrastive and supervised contrastive learning) followed by a clustering step achieved through a semi-supervised K-Means algorithm. Subsequent approaches investigating alternative approaches such as parametric classification \citep{wen2023simgcd}, hierarchal K-Means and self-expertise \citep{RastegarECCV2024}, spatial prompt tuning \citep{wang2024sptnet}, Gaussian Mixture Models \citep{Zhao_2023_ICCV}, student-teacher distillation \citep{vaze2023clevr4}, etc. An in-depth review of the Category Discovery fields was conducted by \citet{CDSurvey}.

There has throughout the GCD research field been an interest in leveraging the underlying hierarchal structure of the dataset in various ways. The SelEx \citep{RastegarECCV2024} method constructed hierarchies through a balanced K-Means approach, SEAL \citep{He2025SEAL} 
proposed a semantic-aware hierarchical learning approach, while InfoSieve \citep{rastegar2023learn} constructed an implicit hierarchy based through an information theoretic lens. 
More recently, two approaches have been proposed which adapt hyperbolic geometry. Firstly, HIDISC \citep{HIDISC} was proposed as a hyperbolic approach for the sub-task of Domain Generalization based on applying CutMix \citep{yun2019cutmix} in the tangent space, Busemann Learning \citep{NEURIPS2021_01259a0c} and hyperbolic contrastive learning. Secondly, HypCD \citep{Liu2025HypCD} explored the use of hyperbolic geometry for GCD. They proposed hyperbolic variants to the original GCD method \citep{Vaze_2022_CVPR}, SimGCD \citep{wen2023simgcd}, and SelEx \citep{RastegarECCV2024}, and trained using both angular and distance based contrastive losses. Through all their experiments they showed clear improvements over the Euclidean variants on standard benchmarks such as the Semantic Shift Benchmark \citep{vaze2022the}.

However, while the HypCD work considers hyperbolic representation learning and parametric classification, the non-parametric clustering approaches are all applied in Euclidean geometry. Therefore, we investigate in depth the effect of performing K-Means clustering in hyperbolic space.

\section{\uppercase{Hyperbolic Geometry}}
Hyperbolic geometry can be represented using different models, with each their own properties. The two models used in this paper are the Lorentz Hyperboloid Model \citep{pmlr-v202-desai23a,ratcliffe1994foundations} for representation learning, and the Klein Model \citep{mao2024kleinmodelhyperbolicneural,ratcliffe1994foundations} for its closed form Einstein Midpoint.
\subsection{The Lorentz Hyperboloid Model}

\begin{figure}[t]
    \centering
    \includegraphics[width=\linewidth]{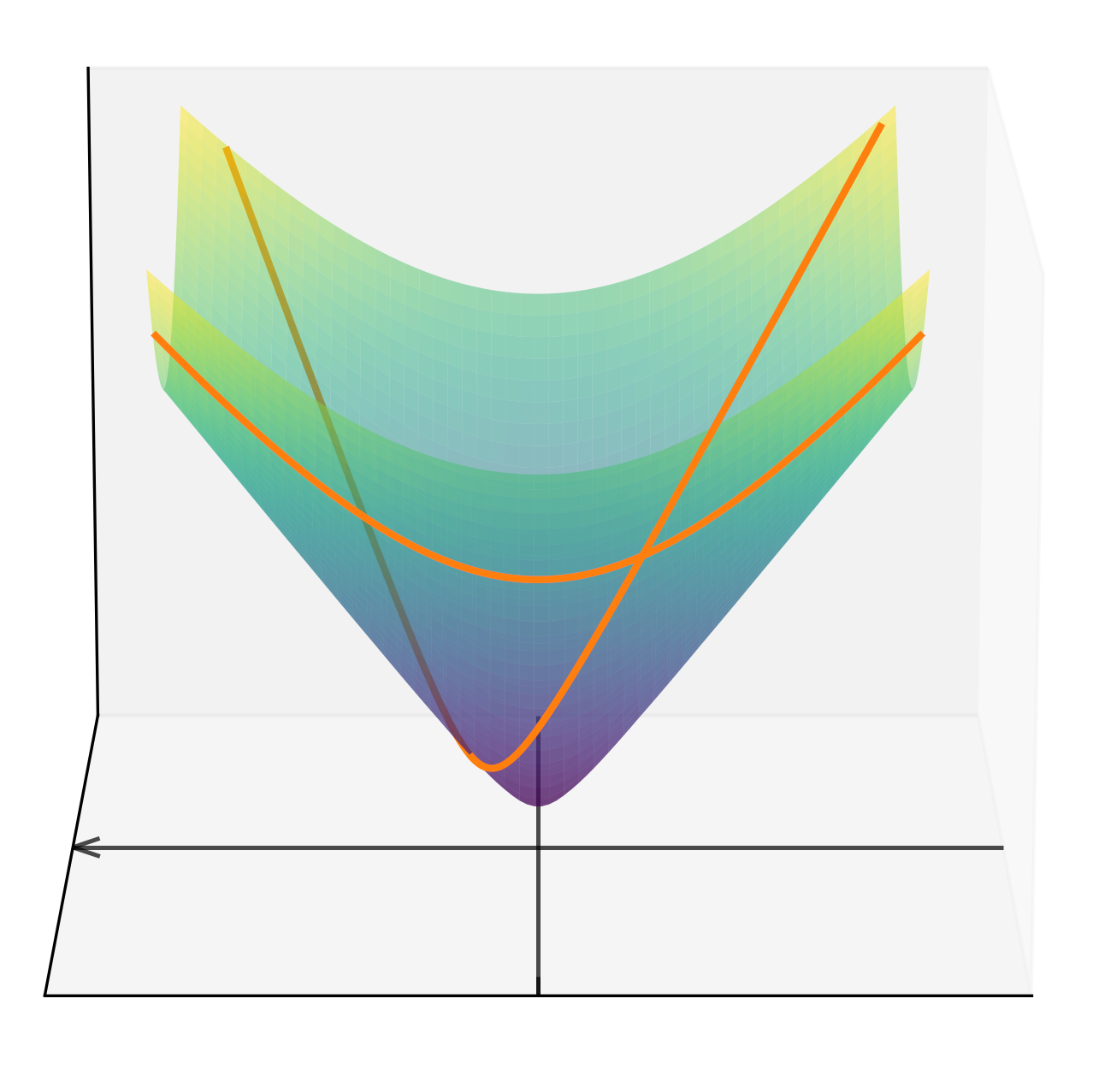}
    \caption{The Lorentz Hyperboloid model $\mathbb{H}^2_1$ as the upper sheet of a hyperboloid. Shown also are geodesics in this model.}
    \label{fig:Lorentz-model}
\end{figure}

Hyperbolic geometry in the Lorentz Hyperboloid model $\mathbb{H}^n_{\kappa}$ is represented as the upper sheet of a hyperboloid in $\mathbb{R}^{n+1}$:

\begin{equation}
    \mathbb{H}_{\kappa}^n=\left\{\mathbf{x}\in\mathbb{R}^{n+1}|\langle\mathbf{x,x}\rangle_{\mathbb{L}}=-\frac{1}{\kappa},x_0>0\right\}
\end{equation}

Where $\langle\mathbf{x,x}\rangle_{\mathbb{L}}=-x_0^2+x_1^2+\cdots +x_n^2$ is the Lorentz inner product and $\kappa=-c$, where c is the curvature. Figure \ref{fig:Lorentz-model} shows $\mathbb{H}^2_1$ as an example.  Borrowing from special relativity theory, a point in the Lorentz Hyperboloid model can be written as $\mathbf{x}_{\mathbb{H}}=[x_{time};\mathbf{x}_{space}]$. Representing the axis of symmetry of the hyperboloid as the time dimension and allowing the Lorentz inner product to be simplified to:
\begin{equation}
    \langle\mathbf{x}_{\mathbb{H}},\mathbf{y}_{\mathbb{H}}\rangle_{\mathbb{L}}=-x_{time}y_{time}+\langle\mathbf{x}_{space},\mathbf{y}_{space}\rangle
\end{equation}
Where, $\langle\mathbf{x},\mathbf{y}\rangle$ is the Euclidean dot product. Additionally, due to the constraint $\langle\mathbf{x,x}\rangle_{\mathbb{L}}=-\frac{1}{\kappa}$ it is possible to calculate the value of $x_{time}$ from $\mathbf{x}_{space}$:
\begin{equation}
    x_{time}=\sqrt{1/\kappa+\langle\mathbf{x}_{space},\mathbf{x}_{space}\rangle}
\label{eq:Lorentz-Time}
\end{equation}
Furthermore, the distance between two points on the Lorentz Hyperboloid is given as:

\begin{equation}
    d_{\mathbb{H}}(\mathbf{x},\mathbf{y})=\sqrt{\frac{1}{\kappa}\cosh^{-1}(-\kappa\langle\mathbf{x},\mathbf{y}\rangle_{\mathbb{L}})}
\label{eq:Lorentz-Dist}
\end{equation}

And lastly, the exponential map can be used to map vectors from the Euclidean tangent space to the Lorentz Hyperboloid:
\begin{equation}
    \mathbf{x}_{space}={expm}(\mathbf{v})=\frac{\sinh(\sqrt{\kappa}\|\mathbf{v}\|)}{\sqrt{\kappa}\|\mathbf{v}\|}\mathbf{v}
\label{eq:Lorentz-exp-map}
\end{equation}

Where $\mathbf{v}\in\mathbb{R}^n$ is a point in Euclidean geometry, and $\|\mathbf{v}\|$ is the L2 norm. While the exponential map only calculates $\mathbf{x}_{space}$, Equation \ref{eq:Lorentz-Time} can be used to find $x_{time}$.

\subsection{The Klein Model}

\begin{figure}[t]
    \centering
    \includegraphics[width=\linewidth]{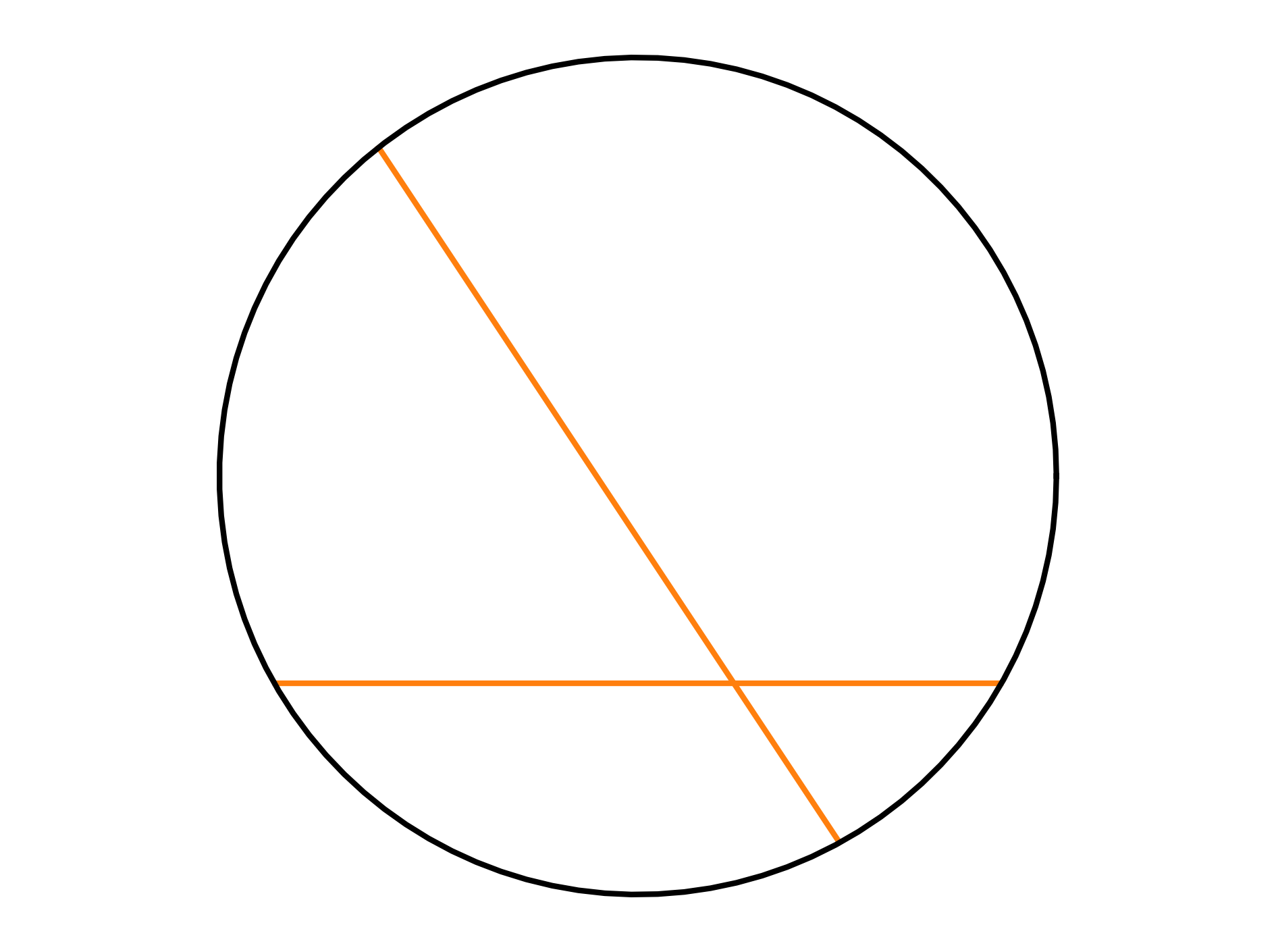}
    \caption{The Klein model $\mathbb{K}^2_1$ as a circle with radius under $1/\kappa$. This model is the projection of the Lorentz Hyperboloid $\mathbb{H}^2_1$ from Figure \ref{fig:Lorentz-model}, which shows how the geodesics project into straight lines.}
    \label{fig:Klein-model}
\end{figure}

The Klein model is constructed by projecting the Lorentz Hyperboloid in $\mathbb{R}^{n+1}$ into an n-dimensional hypersphere of radius $1/\sqrt{\kappa}$ with center at the origo of the hyperboloid $\mathbf{O}_{\mathbb{L}}=[1/\sqrt{\kappa};\mathbf{0}]$:
\begin{equation}
    \mathbf{x}_{\mathbb{K}}=\frac{\mathbf{x}_{space}}{\sqrt{\kappa}x_{time}}
    \label{eq:Lorentz-Klein-Map}
\end{equation}
Where $\mathbf{x}_{\mathbb{K}}$ is a point in the Klein model. This construction makes it isometric to the Lorentz Hyperboloid Model. Being a hypersphere, the set of points in the Klein Model are:
\begin{equation}
    \mathbb{K}^n_{\kappa}=\left\{\mathbf{x}\in\mathbb{R}^n;\|\mathbf{x}\|^2<\frac{1}{\kappa}\right\}
\end{equation}

A special property of the Klein model is that its geodesics are straight lines in $\mathbb{R}^n$, as can be seen in Figure \ref{fig:Klein-model}. This makes it straightforward to define a midpoint in the Klein model \citep{mettes2024hyperbolic}, with one such midpoint being the Einstein Midpoint:
\begin{equation}
    \mu_{\mathbb{K}}=\frac{\sum_{i=1}{\gamma_i\mathbf{x}_{\mathbb{K},i}}}{\sum_{i=1}{\gamma_i}}
\label{eq:Einstein-Midpoint}
\end{equation}

Where $\gamma_i$ are the Lorentz factors:
\begin{equation}
    \gamma_i=\frac{1}{\sqrt{1-\kappa\|\mathbf{x}_{\mathbb{K},i}\|^2}}
\end{equation}

Lastly, the exponential map in the Klein model is:
\begin{equation}
    \mathbf{x}_{\mathbb{K}}={expm}(\mathbf{v})=\frac{\tanh(\sqrt{\kappa}\|\mathbf{v}\|)}{\sqrt{\kappa}\|\mathbf{v}\|}\mathbf{v}, \mathbf{v}\in\mathbb{R}^n
\end{equation}

\section{\uppercase{Hyperbolic K-Means}}

The K-Means clustering algorithm requires a distance function and a centroid. In Euclidean K-Means the Euclidean distance and the arithmetic mean are used. However, in order to use K-Means in hyperbolic geometry, an equivalent distance function and centroid are needed.

For points on the Lorentz Hyperboloid, the distance function $d_{\mathbb{H}}$ from Equation \ref{eq:Lorentz-Dist} can be used. As for the centroid, one option is the Einstein Midpoint in the Klein model found in Equation \ref{eq:Einstein-Midpoint}, which would require mapping between the Lorentz Hyperboloid and Klein Hypersphere. Another option is the Lorentz Centroid \citep{pmlr-v97-law19a} computed directly from points on the Lorentz Hyperboloid:
\begin{equation}
    \mu_{\mathbb{H}}=\frac{1}{\sqrt{\kappa}}\frac{\sum_{i=1}{w_i\mathbf{x}_{\mathbb{H},i}}}{|\|\sum_{i=1}{w_i\mathbf{x}_{\mathbb{H},i}}\|_{\mathbb{L}}|}
\end{equation}
Where $w_i$ are arbitrary sample weights and $\|\mathbf{x}\|_{\mathbb{L}}=\sqrt{\langle \mathbf{x},\mathbf{x}\rangle}_{\mathbb{L}}$ is the Lorentz norm. The Lorentz Centroid minimizes the squared Lorentzian distance between the centroid and all points:
\begin{equation}
    d_{\mathbb{L}}^2(\mathbf{\mu},\mathbf{x})=\|\mathbf{\mu}-\mathbf{x}\|_{\mathbb{L}}^2=-2\frac{1}{\sqrt{\kappa}}-2\langle\mathbf{\mu},\mathbf{x}\rangle_{\mathbb{L}}
\end{equation}

Which does not trivially translate to minimizing the squared Lorentz Hyperboloid distance:
\begin{equation}
    d_{\mathbb{H}}^2(\mathbf{\mu},\mathbf{x})=\frac{1}{\kappa}\cosh^{-1}(-\kappa\langle\mathbf{\mu},\mathbf{x}\rangle_{\mathbb{L}})
\label{eq:Lorentz-Dist-Squared}
\end{equation}
making it uncertain if it can be used as a centroid. However, it is possible to prove that the Einstein Midpoint is the projection of the Lorentz Centroid from the Lorentz Hyperboloid to the Klein Hypersphere:
\begin{lemma}
    The function mapping points from the Klein model to points on the Lorentz Hyperboloid is:
    \begin{equation}
        \pi_{\mathbb{K}\rightarrow\mathbb{H}}(\mathbf{x}_{\mathbb{K}})=\frac{1}{\sqrt{\kappa-\kappa^2\|\mathbf{x}_{\mathbb{K}}\|^2}}[1;\sqrt{\kappa}\mathbf{x}_{\mathbb{K}}]
    \end{equation}
\end{lemma}
\begin{corollary}
    The Lorentz factors in the Einstein Midpoint can be written as:
    \begin{equation}
        \gamma_i=\frac{\sqrt{\kappa}}{\sqrt{\kappa-\kappa^2\|\mathbf{x}_{\mathbb{K}}\|^2}}=\sqrt{\kappa}x_{time}
    \end{equation}
\end{corollary}
The proof for Lemma 1 can be found in the appendix.
\begin{theorem}
    The Einstein Midpoint $\mu_{\mathbb{K}}$ can be found by passing the Lorentz Centroid $\mu_{\mathbb{H}}$ through the map $\pi_{\mathbb{H}\rightarrow\mathbb{K}}$:
    \begin{equation}
        \mu_{\mathbb{K}}=\pi_{\mathbb{H}\rightarrow\mathbb{K}}(\mu_{\mathbb{H}})
    \end{equation}
\label{th:Midpoint-Equivalence}
\end{theorem}
\begin{proof}
    The Lorentz centroid can be split into its separate space and time components as:
    \begin{gather}
        \mu_{time}=\frac{1}{\sqrt{\kappa}}\frac{\sum_{i=1}{w_ix_{time}}}{|\|\sum_{i=1}{w_i\mathbf{x}_{\mathbb{H},i}}\|_{\mathbb{L}}|}\\
        \mu_{space}=\frac{1}{\sqrt{\kappa}}\frac{\sum_{i=1}{w_i\mathbf{x}_{space,i}}}{|\|\sum_{i=1}{w_i\mathbf{x}_{\mathbb{H},i}}\|_{\mathbb{L}}|}
    \end{gather}
    Using these components the mapping $\pi_{\mathbb{H}\rightarrow\mathbb{K}}(\mu_{\mathbb{H}})$ evaluates to:
    \begin{equation}
        \pi_{\mathbb{H}\rightarrow\mathbb{K}}(\mu_{\mathbb{H}})=\frac{\mu_{space}}{\sqrt{\kappa}\mu_{time}}=\frac{\sum_{i=1}{w_i\mathbf{x}_{space,i}}}{\sqrt{\kappa}\sum_{i=1}{w_ix_{time}}}
    \end{equation}
    Which is equal to the Einstein Midpoint if $w_i=1,\forall i$:
    \begin{gather}
        \mathbf{\mu}_{\mathbb{K}}=\frac{\sum_{i=1}{\gamma_i\mathbf{x}_{\mathbb{K},i}}}{\sum_{i=1}{\gamma_i}}\\
        \mathbf{\mu}_{\mathbb{K}}=\frac{\sum_{i=1}{\sqrt{\kappa}x_{time,i}\frac{\mathbf{x}_{space, i}}{\sqrt{\kappa}x_{time,i}}}}{\sum_{i=1}{\sqrt{\kappa}x_{time,i}}}\\
        \mathbf{\mu}_{\mathbb{K}}=\frac{\sum_{i=1}{\mathbf{x}_{space,i}}}{\sqrt{\kappa}\sum_{i=1}{x_{time}}}=\pi_{\mathbb{H}\rightarrow\mathbb{K}}(\mu_{\mathbb{H}})
    \end{gather}
\end{proof}
With this proof the two choices become one, and the Lorentz Centroid can be used as a centroid for the hyperbolic K-Means algorithm without having to transform between hyperbolic models. Furthermore, this proof makes it possible to create an equivalent algorithm for the Poincaré model, used in an ablation study in Section \ref{subsec:ablation}, by utilizing the Einstein Midpoint as a centroid. Figure \ref{fig:Lorentz-KMeans} shows an example of the hyperbolic semi-supervised K-Means algorithm in $\mathbb{H}^2_1$.

\section{\uppercase{Method}}

\subsection{GCD Task Definition}
Generalized Category Discovery (GCD) is an open world classification task, where an algorithm has to cluster a dataset with points belonging to either seen or unseen classes. This is achieved by splitting a dataset into two subsets, an unlabeled subset $D_{\mathcal{U}}:\{x_i, y_i\}\in X, Y_{\mathcal{U}}$ and a labeled subset $D_{\mathcal{L}}:\{x_i, y_i\}\in X,Y_{\mathcal{L}}$ where $ Y_{\mathcal{L}}\subset Y_{\mathcal{U}}$. The training set includes samples from both $D_{\mathcal{L}}$ and $D_{\mathcal{U}}$, but only labels from $D_{\mathcal{L}}$, resulting in a need for a combination of unsupervised and supervised learning. During testing, the algorithm needs to classify samples belonging to both seen and unseen classes \citep{Vaze_2022_CVPR}.

\subsection{HC-GCD}
\begin{figure*}[t]
    \centering
    \includegraphics[width=\linewidth]{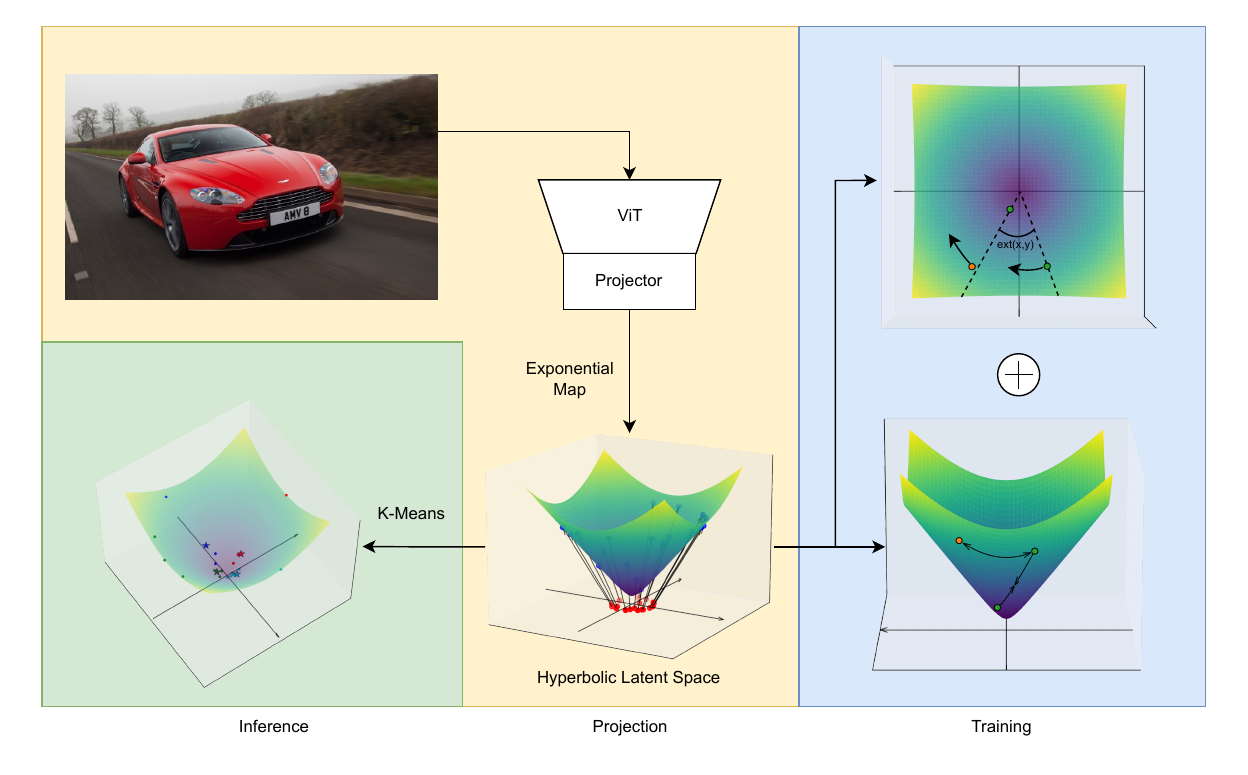}
    \caption{Full pipeline of HC-GCD. Training uses combination of distance and angle based contrastive loss. After training is finished, a hyperbolic semi-supervised K-Means algorithm is used for clustering the embeddings in the Lorentz Hyperboloid.}
    \label{fig:HC-GCD-Inference}
\end{figure*}

The Hyperbolic Clustered GCD (HC-GCD) model closely follows the HypCD implementation of the GCD model (Hyp-GCD) by \citet{Liu_2025_CVPR}. The backbone uses a Vision Transformer (ViT) \citep{ViT-Paper} and a projector head, followed by clipping the norm of the Euclidean embeddings \citep{Guo_2022_CVPR} and transformation into hyperbolic space using an exponential map. While Hyp-GCD transforms to the Poincaré Hypersphere, HC-GCD transforms to the Lorentz Hyperboloid.

The training procedure also mirrors that of Hyp-GCD, using a combination of distance based and angle based self-supervised and supervised contrastive losses. Given the contrastive score between two samples $\sigma(i,\mathbf{y}_{\mathbb{H}})$, where $i$ is the index of the sample, and $\mathbf{y}$ is the positive sample. The self-supervised and supervised contrastive losses are:
\begin{gather}
    \mathcal{L}_{\sigma,i}^u=-\log{(\sigma(i,\mathbf{z'}_{\mathbb{H},i}))}\\
    \mathcal{L}_{\sigma,i}^s=-\frac{1}{|\mathcal{N}(i)|}\sum_{q\in \mathcal{N}(i)}{\log{(\sigma(i,\mathbf{z}_{\mathbb{H},q}))}}
\end{gather}
Where $\mathbf{z}_{\mathbb{H},i}$ and $\mathbf{z'}_{\mathbb{H},i}$ are embeddings from two views, where each view is an augmentation of sample $i$, and $\mathcal{N}(i)$ is a set of indices belonging to the same class as index $i$.
For the distance based contrastive loss, the contrastive score uses the negative hyperbolic distance as a similarity metric:
\begin{equation}
    \sigma_{D}(i,\mathbf{y})=\frac{\exp(-d_{\mathbb{H}}(\mathbf{z}_{\mathbb{H},i}\mathbf{y})/\tau)}{\sum_{n}{\mathbf{1}_{[n\neq i]}\exp(-d_{\mathbb{H}}(\mathbf{z}_{\mathbb{H},i}\mathbf{z}_{\mathbb{H},n})/\tau)}}
\end{equation}
Where $\tau$ is a temperature constant and $\mathbf{1}_{[n\neq i]}$ is an indicator function that evaluates to 1 if and only if $n\neq i$.

For the angle based losses the exterior angle from MERU \citep{pmlr-v202-desai23a} is utilized:
\begin{gather}
    ext(\mathbf{x},\mathbf{y})=\cos^{-1}\left(\frac{y_{time}+x_{time}\kappa\langle\mathbf{x},\mathbf{y}\rangle_{\mathbb{L}}}{\|\mathbf{x}_{space}\|\sqrt{(\kappa\langle\mathbf{x},\mathbf{y}\rangle_{\mathbb{L}})^2-1}}\right)\\
    \sigma_A(i,\mathbf{y})=\frac{\exp(ext(\mathbf{z}_{\mathbb{H},i}\mathbf{y})/\tau)}{\sum_{n}{\mathbf{1}_{[n\neq i]}\exp(ext(\mathbf{z}_{\mathbb{H},i}\mathbf{z}_{\mathbb{H},n})/\tau)}}
\end{gather}

The total training loss is a weighted combination of all four losses:
\begin{gather}
    \mathcal{L}^s=((1-\alpha)\mathcal{L}^s_{\sigma_D}+\alpha\mathcal{L}^s_{\sigma_A})\\
    \mathcal{L}^u=((1-\alpha)\mathcal{L}^u_{\sigma_D}+\alpha\mathcal{L}^u_{\sigma_A})\\
    \mathcal{L}=(1-\lambda)\mathcal{L}^s+\lambda\mathcal{L}^u
\end{gather}
With $\lambda$ being the weighting hyperparameter between self-supervised and supervised contrastive loss, and $\alpha$ being the weighting hyperparameter between distance and angle-based losses. Similar to Hyp-GCD the $\alpha$ weight is linearly decayed from 1 to 0 as training progresses.

After finished training, the hyperbolic K-Means algorithm is used to cluster the embeddings in the Lorentz Hyperboloid present by the model. Figure \ref{fig:HC-GCD-Inference} shows the overall HC-GCD model pipeline. 

\section{\uppercase{Experimental Design}}

The experiments were carried on the Semantic Shift Benchmark (SSB) datasets \citep{vaze2022the}, namely CUB \citep{wah2011caltech}, Stanford Cars \citep{Krause_2013_ICCV_Workshops} and FGVC-Aircraft \citep{maji2013fine}. Following the original GCD setup, half of the classes are assigned as seen classes $Y_{\mathcal{L}}$. Thereafter, the labeled dataset $D_{\mathcal{L}}$ is comprised from 50\% of the samples belonging to the seen classes. The remaining samples from both seen and unseen classes create the unlabeled dataset $D_{\mathcal{U}}$ \citep{Vaze_2022_CVPR}.

The HC-GCD model uses a ViT-B14 \citep{ViT-Paper} pre-trained with DINOv2 \citep{oquab2024dinov2learningrobustvisual} as baseline followed by a projection head comprised of 4 linear layers and GELU \citep{hendrycks2023gaussianerrorlinearunits} activations according to this setup:
\begin{equation}
    \mathbb{R}^I\xrightarrow{GELU}\mathbb{R}^{2048}\xrightarrow{GELU}\mathbb{R}^{2048}\xrightarrow{GELU}\mathbb{R}^{256}\rightarrow\mathbb{R}^{256}
\end{equation}
Where $I$ is the input dimension for embeddings from the ViT. The Euclidean embeddings outputted by the projection head are then clipped to a norm of 2.3. These clipped embeddings are then transformed to the Lorentz Hyperboloid with a fixed curvature of $c=-0.05$ using the exponential map. These two values are chosen because they were shown to perform best for the SSB datasets in the HypCD paper \citep{Liu_2025_CVPR}

During training, gradient clipping is used to avoid exploding gradients. The gradients are clipped to a max absolute value of 1.0, then the gradients are scaled to have a max absolute average of 0.25. Furthermore, a cosine annealing scheduler is used, with a maximum learning rate of 0.1 and a minimum of 0.0001 with a SGD optimizer. 

For the losses a temperature value of $\tau=0.07$ is used and the weight factor $\alpha$ is linearly decayed from 1 to 0, while the other weight factor is $\lambda=0.35$. The model is trained with a batch size of 128 and for 200 epochs and all models are trained on an Nvidia L40s GPU.

The model with the lowest loss is clustered using both hyperbolic K-Means and Euclidean K-Means. During Euclidean evaluation the projection head is removed, and embeddings directly from the ViT are clustered. The results from the K-Means clustering are then evaluated using clustering accuracy, which was defined by Vaze et al. \citep{Vaze_2022_CVPR}. Three accuracy metrics are provided, "All" accuracy calculated using all classes, "Old" accuracy calculated only using seen classes and "New" accuracy calculated only using unseen classes. The full training set is used for clustering, while accuracy is evaluated on only the unlabeled training set as per convention within the GCD field \citep{Vaze_2022_CVPR, Liu2025HypCD, wen2023simgcd}.

\begin{table*}[t]
\caption{Mean and standard deviation of the accuracy of the GCD models on the fine-grained SSB datasets. Lorentz contrastive training refers to HC-GCD while Poincaré refers to Hyp-GCD. For the HC-GCD, performance is reported using both the Euclidean and hyperbolic K-Means algorithms. The highest mean accuracy for each column is highlighted in bold, while the second highest is underlined, while the equivalent is done for the lowest standard deviations.}
\label{tab:hgcd-performance}
\resizebox{\linewidth}{!}{
\begin{tabular}{c|c|ccccccccc}
\toprule
Contrastive & K-Means   & \multicolumn{3}{c}{CUB} & \multicolumn{3}{c}{Stanford Cars} & \multicolumn{3}{c}{FGVC-Aircraft} \\
Training    & Algorithm & All & Old & New & All & Old & New & All & Old & New \\
\midrule
\multicolumn{11}{c}{Mean (Higher is Better)}\\
\midrule
Lorentz & Lorentz & \textbf{71.70} & \textbf{76.34} & 69.38 & \textbf{73.69} & \textbf{82.41} & \textbf{69.49} & \textbf{61.90} & \underline{61.85} & \textbf{61.93} \\
Lorentz & Euclidean & 70.22 & 70.22 & \underline{70.22} & \underline{70.74} & \underline{79.09} & \underline{66.71} & \underline{61.31} & \textbf{64.31} & \underline{59.81} \\
Poincaré   & Euclidean & \underline{71.52} & \underline{72.56} & \textbf{71.00} & 68.31 & 75.46 & 64.85 & 59.36 & 59.38 & 59.34 \\
\midrule
\multicolumn{11}{c}{Standard Deviation (Lower is Better)}\\
\midrule
Lorentz & Lorentz & \textbf{00.55} & \underline{00.93} & \textbf{01.24} & \textbf{00.37} & \textbf{00.57} & \textbf{00.71} & \underline{02.49} & 04.04 & \textbf{01.79} \\
Lorentz & Euclidean & \underline{01.29} & \textbf{00.74} & \underline{01.91} & \underline{00.81} & 00.90 & \underline{00.78} & 02.65 & \underline{02.56} & 02.69 \\
Poincaré   & Euclidean & 02.06 & 01.36 & 02.62 & 01.25 & \underline{00.61} & 01.63 & \textbf{02.36} & \textbf{02.49} & \underline{02.30} \\
\bottomrule
\end{tabular}
}
\end{table*}
\section{\uppercase{Results}}
To compare our results, we trained and evaluated Hyp-GCD locally as a baseline.
Table \ref{tab:hgcd-performance} shows the mean accuracy from three different seeds. HC-GCD outperforms in two out of three datasets showing potential in using the Lorentz Hyperboloid model for the GCD task. Furthermore, using the hyperbolic K-Means algorithm leads to improvement on almost all metrics, showing its superiority to Euclidean K-Means when training in the Lorentz Hyperboloid model of hyperbolic geometry.

Table \ref{tab:hgcd-performance} also shows the standard deviation of the reported accuracies. This table shows improved training stability when using the Lorentz Hyperboloid model, as accuracies across seeds are more consistent.

\subsection{Ablations}\label{subsec:ablation}

\begin{table*}[t]
\caption{Mean accuracy of the HC-GCD models when trained with a Euclidean clipping of $r=2.3$ and without any Euclidean clipping. The highest accuracy for each column is highlighted in bold, while the second highest is underlined.}
\label{tab:hgcd-Clipping-Lorentz}
\resizebox{\linewidth}{!}{
\begin{tabular}{c|c|ccccccccc}
\toprule
Euclidean & K-Means   & \multicolumn{3}{c}{CUB} & \multicolumn{3}{c}{Stanford Cars} & \multicolumn{3}{c}{FGVC-Aircraft} \\
Clipping    & Algorithm & All & Old & New & All & Old & New & All & Old & New \\
\midrule
\ding{51} & Lorentz   & \textbf{71.70} & \textbf{76.34} & \underline{69.38} & \textbf{73.69} & \underline{82.41} & \textbf{69.49} & 61.90 & 61.85 & \textbf{61.93} \\
\ding{55} & Lorentz   & 68.88 & \underline{73.34} & 66.66 & 70.20 & \textbf{84.69} & 63.20 & \textbf{64.72} & \textbf{75.17} & 59.59 \\
\ding{51} & Euclidean & \underline{70.22} & 70.22 & \textbf{70.22} & \underline{70.74} & 79.09 & \underline{66.71} & 61.31 & 64.31 & 59.81 \\
\ding{55} & Euclidean & 69.28 & 71.16 & 68.34 & 68.77 & 80.19 & 63.25 & \underline{64.18} & \underline{72.15} & \underline{60.20} \\
\bottomrule
\end{tabular}
}
\end{table*}
\paragraph{Euclidean Clipping} \citet{Guo_2022_CVPR} show that clipping the Euclidean embeddings leads to improved performance on models trained in the Poincaré Hypersphere, as unclipped embeddings lead to vanishing gradient problems. However, this does not necessarily apply to models trained in the Lorentz Hyperboloid. Furthermore, \citet{pmlr-v202-desai23a} do not utilize embedding clipping when training MERU. Therefore, an ablation study is carried to investigate the impact of using clipped embeddings when training in the Lorentz Hyperboloid.

Table \ref{tab:hgcd-Clipping-Lorentz} shows the results of this ablation. These results show that the performance from using embedding clipping is dataset dependant, as tests on FGVC-Aircraft gain an accuracy improvement when clipping is not used. Furthermore, a pattern can be seen with the old and new accuracy, as removing embedding clipping leads to increased accuracy on old classes, and a decreased accuracy on new classes most of the times.

\begin{table*}[t]
\caption{Mean accuracy of models trained on a ViT with registers backbone and a projector without its last layer. Lorentz contrastive training refers to HC-GCD while Poincaré refers to Hyp-GCD. Percent difference with models trained without registers and with the last layer of the projector present are reported under each result. The highest accuracy for each column is highlighted in bold, while the second highest is underlined.}
\label{tab:ablation-registers-mean}
\resizebox{\linewidth}{!}{
\begin{tabular}{c|c|ccccccccc}
\toprule
Contrastive & K-Means   & \multicolumn{3}{c}{CUB} & \multicolumn{3}{c}{Stanford Cars} & \multicolumn{3}{c}{FGVC-Aircraft} \\
Training    & Algorithm & All & Old & New & All & Old & New & All & Old & New \\
\midrule
Lorentz & Lorentz & \underline{71.59} & \textbf{77.65} & 68.56 & \textbf{74.79} & \underline{81.39} & \textbf{71.60} & 63.08 & \underline{65.75} & 61.74 \\
 & & (-0.11) & (+1.31) & (-0.82) & (+1.10) & (-1.02) & (+1.11) & (+1.18) & (+3.90) & (-0.19) \\
Lorentz & Euclidean & 69.27 & 69.05 & \underline{69.38} & \underline{73.46} & 80.29 & \underline{70.17} & \underline{63.12} & 61.53 & \underline{63.91} \\
 & & (-0.95) & (-1.17) & (-0.84) & (+2.72) & (+1.20) & (+3.46) & (+1.81) & (-2.78) & (+4.10) \\
Poincaré   & Euclidean & \textbf{72.35} & \underline{72.31} & \textbf{72.37} & 72.66 & \textbf{81.51} & 68.38 & \textbf{64.92} & \textbf{65.95} & \textbf{64.40} \\
 & & (+0.83) & (-0.25) & (+1.37) & (+4.35) & (+6.05) & (+3.53) & (+5.56) & (+6.57) & (+5.06) \\
\bottomrule
\end{tabular}
}
\end{table*}

\begin{table}[t]
\caption{Mean homogeneity of the GCD models on the FGVC Aircraft dataset when varying the ground truth label granularity. Lorentz contrastive training refers to HC-GCD while Poincaré refers to Hyp-GCD. The highest homogeneity for each column is highlighted in bold, while the second highest is underlined.}
\label{tab:hgcd-homogeneity-mean}
\resizebox{\linewidth}{!}{
\begin{tabular}{c|c|ccc}
\toprule
Contrastive & K-Means   & \multicolumn{3}{c}{FGVC-Aircraft Labels}  \\
Training    & Algorithm & Manufacturer & Family & Variant\\
\midrule
Lorentz & Lorentz & \textbf{90.46} & \textbf{89.09} & \textbf{82.48} \\
Lorentz & Euclidean & \underline{89.99} & \underline{88.62} & \underline{82.41} \\
Poincaré & Euclidean & 87.14 & 85.78 & 80.09 \\
\bottomrule
\end{tabular}
}
\end{table}
\paragraph{ViT with Registers} The code provided by \citet{Liu2025HypCD} for their Hyp-GCD model differed from the original GCD model in two ways:
\begin{itemize}
    \item The ViT model utilizes registers, following the setup by \citet{darcet2024visiontransformersneedregisters}.
    \item The projection MLP is missing the last layer, making it similar to the one used in SimGCD by \citet{wen2023simgcd}:
    \begin{equation}
    \mathbb{R}^I\xrightarrow{GELU}\mathbb{R}^{2048}\xrightarrow{GELU}\mathbb{R}^{2048}\xrightarrow{GELU}\mathbb{R}^{256}
    \end{equation}
\end{itemize}
Therefore, an ablation study is performed to evaluate the impact of these changes on the HC-GCD model. Table \ref{tab:ablation-registers-mean} compares the results from the original setup, and the setup using ViT with registers. It can be seen that the Hyp-GCD training setup gains accuracy in all datasets when ViT with registers is used, especially with the Stanford Cars and FGVC-Aircraft datasets. Lorentz training also benefits from using ViT with registers, however, it is not the case in all datasets, as CUB performs better without registers. Regardless, Lorentz training benefits from using the hyperbolic K-Means algorithm regardless of the backbone used.

\paragraph{Consistency Across Label Granularity} A key motivation for using hyperbolic geometry is that the space more naturally encodes hierarchal structures. We investigate this by varying the label granularity and measuring the number of elements in the predicted clusters which belong to ground truth labels. We determine this using the FGVC Aircraft dataset, constructing coarse to fine-grained labels by using the Manufacturer, Family, and Variant labels for data points. The Homogeneity \citep{Homogeneity} metric is used to quantify the clustering performance, as shown in Table \ref{tab:hgcd-homogeneity-mean}. We find that HC-GCD generates more consistent clusters across all label granularities, outperforming Hyp-GCD by several percentage points. We also find that performing K-Means in hyperbolic space leads to the best performance across all granularity levels, indicating a clear benefit of clustering in hyperbolic space.

\paragraph{Hyperbolic K-Means in Poincaré} An ablation is devised to test the performance of performing K-Means on hyperbolic embeddings from the Hyp-GCD model. To achieve that, a hyperbolic K-Means algorithm for the Poincaré model is used based on the Poincaré distance and the Einstein Midpoint as a centroid. To calculate the Einstein Midpoint, the embeddings in the Poincaré model are transformed to the Klein model, then the calculated midpoint is transformed back to the Poincaré model. This is chosen due to the fact that there are no closed form midpoints in the Poincaré model.

It is possible to compare the hyperbolic K-Means between the Lorentz Hyperboloid model and the Poincaré model due to the distance and centroid being equivalent. The distance functions are equivalent due to the two models being isometric \citep{ratcliffe1994foundations}, and the centroids are equivalent due to the Einstein Midpoint corresponding to the Lorentz Centroid as per Theorem \ref{th:Midpoint-Equivalence}.

We find that clustering in Poincaré space results in close to zero accuracy on all datasets. Upon further inspection we discovered that all samples were always assigned to a single centroid. Therefore, we conclude that further work is needed in order for K-Means in Poincaré to be made feasible. 

\section{\uppercase{Limitations}}
In this paper we have demonstrated the feasibility and benefits of performing clustering in hyperbolic space. However, our work is not without limitations. We specifically highlight two key limitations.
Firstly, our work is missing an adaptation of leading non-parametric methods within the GCD field, such as Hyp-SelEx. These were excluded as we focused on isolating the effect of hyperbolic clustering, with no bells or whistles.
Secondly, we have not demonstrated the effectiveness of hyperbolic clustering in the Poincaré hyperbolic space. This warrants a deeper future study into the feasibility of hyperbolic clustering across different hyperbolic geometry models.

\section{\uppercase{Conclusion}}

Throughout this paper we have investigated whether using hyperbolic representation learning and clustering in the Lorentz Hyperboloid model is beneficial for the GCD task. This is due to the fact that previous hyperbolic methods choose to remove the hyperbolic exponential map during inference and cluster directly on Euclidean embeddings, which may lead to misrepresentation of the learned hierarchical structures. To that end, we adapt the K-Means algorithm to hyperbolic geometry, and adapt the original GCD model by \citet{Vaze_2022_CVPR} creating the Hyperbolic Clustered GCD (HC-GCD) method. We test our method using both hyperbolic and Euclidean K-Means, while comparing to the Hyp-GCD model by Liu et al, which trains in the Poincaré model of hyperbolic geometry.

Our results show that using the Lorentz Hyperboloid model for representation learning performs on par with Hyp-GCD, while also showing that hyperbolic clustering is necessary for improved performance when training in the Lorentz Hyperboloid model. Furthermore, our ablations showed that clustering directly in hyperbolic space leads to more consistent predictions across label granularity, as the model trained in the Lorentz Hyperboloid and clustered with hyperbolic K-Means achieved the highest homogeneity on all FGVC-Aircraft label classes.

\section*{\uppercase{Acknowledgments}} This work was supported by the Pioneer Centre for AI (DNRF grant number P1).

\bibliographystyle{apalike}
{\small
\bibliography{example.bib}}

\section*{\uppercase{Appendix}}

\setcounter{lemma}{0}
\begin{lemma}
    The function mapping points from the Klein model to points on the Lorentz Hyperboloid is:
    \begin{equation}
        \pi_{\mathbb{K}\rightarrow\mathbb{H}}(\mathbf{x}_{\mathbb{K}})=\frac{1}{\sqrt{\kappa-\kappa^2\|\mathbf{x}_{\mathbb{K}}\|^2}}[1;\sqrt{\kappa}\mathbf{x}_{\mathbb{K}}]
    \end{equation}
\end{lemma}

\begin{proof}
    Using the property $x_{time}=\sqrt{1/\sqrt{\kappa}+\|\mathbf{x}_{space}\|^2}$, we can rewrite Equation \ref{eq:Lorentz-Klein-Map}, mapping points from the Lorentz Hyperboloid to the Klein model, as:
    \begin{equation}
        \mathbf{x}_{\mathbb{K}}=\frac{\mathbf{x}_{space}}{\sqrt{\kappa}x_{time}}=\frac{\mathbf{x}_{space}}{\sqrt{\kappa}\sqrt{1/\kappa+\|\mathbf{x}_{space}\|^2}}
    \label{eq:KL-Map-Proof1}
    \end{equation}
    On the other hand, by rearranging the same mapping function, we get:
    \begin{equation}
        \mathbf{x}_{space}=\sqrt{\kappa}x_{time}\mathbf{x}_{\mathbb{K}}
    \label{eq:KL-Map-Proof2}
    \end{equation}
    By replacing $\mathbf{x}_{space}$ in Equation \ref{eq:KL-Map-Proof1} with the value in Equation \ref{eq:KL-Map-Proof2} we get:
    \begin{gather}
        \mathbf{x}_{\mathbb{K}}=\frac{\sqrt{\kappa}x_{time}}{\sqrt{\kappa}\sqrt{1/\kappa+\|\sqrt{\kappa}x_{time}\mathbf{x}_{\mathbb{K}}\|^2}}\mathbf{x}_{\mathbb{K}}\\
        \mathbf{x}_{\mathbb{K}}=\frac{x_{time}}{\sqrt{1/\kappa+\kappa x^2_{time}\|\mathbf{x}_{\mathbb{K}}\|^2}}\mathbf{x}_{\mathbb{K}}
    \end{gather}
    Which means that:
    \begin{gather}
        1=\frac{x_{time}}{\sqrt{1/\kappa+\kappa x^2_{time}\|\mathbf{x}_{\mathbb{K}}\|^2}}\\
        1/\kappa+\kappa x^2_{time}\|\mathbf{x}_{\mathbb{K}}\|^2=x^2_{time}\\
        1/\kappa=(1-\kappa\|\mathbf{x}_{\mathbb{K}}\|^2)x^2_{time}\\
        \frac{1}{\kappa(1-\kappa\|\mathbf{x}_{\mathbb{K}}\|^2)}=x^2_{time}
    \end{gather}
    \begin{equation}
        \frac{1}{\sqrt{\kappa-\kappa^2\|\mathbf{x}_{\mathbb{K}}\|^2}}=x_{time}
    \label{eq:x_time_from_Klein}
    \end{equation}
    By inserting the value of $x_{time}$ into Equation \ref{eq:KL-Map-Proof2} we get:
    \begin{equation}
        \mathbf{x}_{space}=\sqrt{\kappa}\frac{1}{\sqrt{\kappa-\kappa^2\|\mathbf{x}_{\mathbb{K}}\|^2}}\mathbf{x}_{\mathbb{K}}
    \label{eq:x_space_from_Klein}
    \end{equation}
    Hence, using the definitions of $x_{time}$ and $\mathbf{x}_{space}$ from Equations \ref{eq:x_time_from_Klein} and \ref{eq:x_space_from_Klein} respectively the mapping function becomes:
    \begin{equation}
        \pi_{\mathbb{K}\rightarrow\mathbb{H}}(\mathbf{x}_{\mathbb{K}})=\frac{1}{\sqrt{\kappa-\kappa^2\|\mathbf{x}_{\mathbb{K}}\|^2}}[1;\sqrt{\kappa}\mathbf{x}_{\mathbb{K}}]
    \end{equation}
\end{proof}

\clearpage
\newpage
\end{document}